\documentclass{article}

\usepackage{microtype}
\usepackage{graphicx}
\usepackage{subcaption}
\usepackage{booktabs} 
\usepackage{braket}
\usepackage{amsfonts,amsmath}
\usepackage{color}
\usepackage{mathtools}
\usepackage[toc,page]{appendix}
\usepackage{amsthm}
\usepackage{wrapfig}
\usepackage{float}

\usepackage{hyperref}

\usepackage[accepted]{icml2018}

\icmltitlerunning{Hyperbolic Entailment Cones}

\begin{document}

\twocolumn[
\icmltitle{Hyperbolic Entailment Cones for Learning Hierarchical Embeddings}




\begin{icmlauthorlist}
\icmlauthor{Octavian-Eugen Ganea}{to}
\icmlauthor{Gary B{\'e}cigneul}{to}
\icmlauthor{Thomas Hofmann}{to}

\end{icmlauthorlist}

\icmlaffiliation{to}{Department of Computer Science, ETH Zurich, Switzerland}

\icmlcorrespondingauthor{Octavian-Eugen Ganea}{octavian.ganea@inf.ethz.ch}
\icmlcorrespondingauthor{Gary B{\'e}cigneul}{gary.becigneul@inf.ethz.ch}
\icmlcorrespondingauthor{Thomas Hofmann}{thomas.hofmann@inf.ethz.ch}

\icmlkeywords{differential geometry, hyperbolic spaces, graph embeddings, hierarchical, representation, tree, poincare, entailment}

\vskip 0.3in
]



\printAffiliationsAndNotice{}  

\newcommand{\mat}[1]{{\mathbf #1}}
\newcommand{\x}{{\mathbf x}}
\newcommand{\I}{{\mathbf I}}
\newcommand{\hyperspace}{{\mathcal X}}
\newcommand{\D}{{\mathbb D}}
\newcommand{\HH}{{\mathbb H}}
\newcommand{\E}{{\mathbb E}}
\newcommand{\M}{{\mathbf M}}
\newcommand{\R}{{\mathbb R}}
\newcommand{\C}{{\mathbb C}}
\newcommand{\z}{{\mathbf z}}
\renewcommand{\Re}{{\mathbb R}}
\newcommand*{\QEDB}{\hfill\ensuremath{\square}}%
\newcommand{\hi}[1]{\textcolor{red}{#1}}

\newtheorem{theorem}{Theorem}
\newtheorem{corollary}{Corollary}[theorem]
\newtheorem{lemma}[theorem]{Lemma}
 
\DeclarePairedDelimiterX{\inp}[2]{\langle}{\rangle}{#1,#2}

\begin{abstract}
Learning graph representations via low-dimensional embeddings that preserve relevant network properties is an important class of problems in machine learning. We here present a novel method to embed directed acyclic graphs. Following prior work, we first advocate for using hyperbolic spaces which provably model tree-like structures better than Euclidean geometry. Second, we view hierarchical relations as partial orders defined using a family of nested geodesically convex cones. We prove that these entailment cones admit an optimal shape with a closed form expression both in the Euclidean and hyperbolic spaces, and they canonically define the embedding learning process. Experiments show significant improvements of our method over strong recent baselines both in terms of representational capacity and generalization.

\end{abstract}

\section{Introduction}

Producing high quality feature representations of data such as text or images is a central point of interest in artificial intelligence. A large line of research focuses on embedding discrete data such as graphs~\citep{grover2016node2vec,goyal2017graph} or linguistic instances \citep{mikolov2013distributed,pennington2014glove,kiros2015skip} into continuous spaces that exhibit certain desirable geometric properties. This class of models has reached state-of-the-art results for various tasks and applications, such as link prediction in knowledge bases~\citep{nickel2011three,bordes2013translating} or in social networks~\citep{hoff2002latent}, text disambiguation~\citep{ganea2017deep}, word hypernymy~\citep{shwartz2016improving}, textual entailment~\citep{rocktaschel2015reasoning} or taxonomy induction~\citep{fu2014learning}. 

Popular methods typically embed symbolic objects in low dimensional Euclidean vector spaces using a strategy that aims to capture semantic information such as functional similarity. Symmetric distance functions are usually minimized between representations of correlated items during the learning process. Popular examples are word embedding algorithms trained on corpora co-occurrence statistics which have shown to strongly relate semantically close words and their topics~\citep{mikolov2013distributed,pennington2014glove}.

However, in many fields (e.g. Recommender Systems, Genomics~\citep{billera2001geometry}, Social Networks), one has to deal with data whose latent anatomy is best defined by non-Euclidean spaces such as Riemannian manifolds~\citep{bronstein2017geometric}. Here, the Euclidean symmetric models suffer from not properly reflecting complex data patterns such as the latent hierarchical structure inherent in taxonomic data. To address this issue, the emerging trend of geometric deep learning\footnote{http://geometricdeeplearning.com/} is concerned with non-Euclidean manifold representation learning.

%
%
%
%
%
%

In this work, we are interested in geometrical modeling of hierarchical structures, directed acyclic graphs (DAGs) and entailment relations via low dimensional embeddings. Starting from the same motivation, the order embeddings method \citep{vendrov2015order} explicitly models the partial order induced by entailment relations between embedded objects. Formally, a vector $x\in\Re^n$ represents a more general concept than any other embedding from the Euclidean entailment region $\mathcal{O}_x:=\lbrace y\mid y_i\geq x_i, \forall 1 \leq i \leq n\rbrace$. A first concern is that the capacity of order embeddings grows only linearly with the embedding space dimension. Moreover, the regions $\mathcal{O}_x$ suffer from heavy intersections, implying that their disjoint volumes rapidly become bounded\footnote{For example, in $n$ dimensions, no $n+1$ distinct regions $\mathcal{O}_x$ can simultaneously have unbounded disjoint sub-volumes.}. As a consequence, representing wide (with high branching factor) and deep hierarchical structures in a bounded region of the Euclidean space would cause many points to end up undesirably close to each other. This also implies that Euclidean distances would no longer be capable of reflecting the original  tree metric.

Fortunately, the hyperbolic space does not suffer from the aforementioned capacity problem because the volume of any ball grows exponentially with its radius, instead of polynomially as in the Euclidean space. This exponential growth property enables hyperbolic spaces to embed any weighted tree while almost preserving their metric\footnote{See end of Section~\ref{ssec:hyperbolic} for a rigorous formulation.} \citep{gromov1987hyperbolic,bowditch2006course,sarkar2011low}. The tree-likeness of hyperbolic spaces has been extensively studied \citep{hamann_2017}. Moreover, hyperbolic spaces are used to visualize large hierarchies~\citep{lamping1995focus+}, to efficiently forward information in complex networks~\citep{krioukov2009greedy,cvetkovski2009hyperbolic} or to embed heterogeneous, scale-free graphs \citep{shavitt2008hyperbolic,krioukov2010hyperbolic,blasius2016efficient}. 


From a machine learning perspective, recently, hyperbolic spaces have been observed to provide powerful representations of entailment relations  \citep{nickel2017poincar}. The latent hierarchical structure surprisingly emerges as a simple reflection of the space's negative curvature.
However, the approach of \citep{nickel2017poincar} suffers from a few drawbacks: first, their loss function causes most points to collapse on the border of the Poincar\'e ball, as exemplified in Figure~\ref{fig:mammals}. Second, the hyperbolic distance alone (being symmetric) is not capable of encoding asymmetric relations needed for entailment detection, thus a heuristic score is chosen to account for concept generality or specificity encoded in the embedding norm. 

We here inspire ourselves from hyperbolic embeddings \citep{nickel2017poincar} and order embeddings \citep{vendrov2015order}. Our contributions are as follows:
\begin{itemize}
\item We address the aforementioned issues of~\citep{nickel2017poincar} and~\citep{vendrov2015order}. We propose to replace the entailment regions $\mathcal{O}_x$ of order-embeddings by a more efficient and generic class of objects, namely \textit{geodesically convex entailment cones}. These cones are defined on a large class of Riemannian manifolds and induce a partial ordering relation in the embedding space. 
\item The optimal entailment cones satisfying four natural properties surprisingly exhibit canonical closed-form expressions in both Euclidean and hyperbolic geometry that we rigorously derive.
\item An efficient algorithm for learning hierarchical embeddings of directed acyclic graphs is presented. This learning process is driven by our entailment cones. 
\item Experimentally, we learn high quality embeddings and improve over experimental results in \citep{nickel2017poincar} and \citep{vendrov2015order} on hypernymy link prediction for word embeddings, both in terms of capacity of the model and generalization performance. 
\end{itemize}

We also compute an analytic closed-form expression for the exponential map in the $n$-dimensional Poincar\'e ball, allowing us to perform full Riemannian optimization \citep{bonnabel2013stochastic} in the Poincar\'e ball, as opposed to the approximate optimization method used by~\citep{nickel2017poincar}.


\section{Mathematical preliminaries}\label{sec:math}
We now briefly visit some key concepts needed in our work. 
\paragraph{Notations.} We always use $\|\cdot \|$ to denote the Euclidean norm of a point (in both hyperbolic or Euclidean spaces). We also use $\inp{\cdot}{\cdot}$ to denote the Euclidean scalar product.

\subsection{Differential geometry}
For a rigorous reasoning about hyperbolic spaces, one needs to use concepts in differential geometry, some of which we highlight here. For an in-depth introduction, we refer the reader to \citep{spivak1979comprehensive} and \citep{hopper2010ricci}. 

\paragraph{Manifold.} A \textit{manifold} $\mathcal{M}$ of dimension $n$ is a set that can be locally approximated by the Euclidean space $\mathbb{R}^n$. For instance, the sphere $\mathbb{S}^2$ and the torus $\mathbb{T}^2$ embedded in $\mathbb{R}^3$ are $2$-dimensional manifolds, also called surfaces, as they can locally be approximated by $\mathbb{R}^2$. The notion of manifold is a generalization of the notion of surface.

\paragraph{Tangent space.} For $x\in\mathcal{M}$, the \textit{tangent space} $T_x\mathcal{M}$ of $\mathcal{M}$ at $x$ is defined as the $n$-dimensional vector-space approximating $\mathcal{M}$ around $x$ at a first order. It can be defined as  the set of vectors $v$ that can be obtained as $v:=c'(0)$, where $c:(-\varepsilon,\varepsilon)\to\mathcal{M}$ is a smooth path in $\mathcal{M}$ such that $c(0)=x$.

\paragraph{Riemannian metric.} A \textit{Riemannian metric} $g$ on $\mathcal{M}$ is a collection $(g_x)_x$ of inner-products $g_x:T_x\mathcal{M}\times T_x\mathcal{M}\to\mathbb{R}$ on each tangent space $T_x\mathcal{M}$, depending smoothly on $x$. Although it defines the geometry of $\mathcal{M}$ locally, it induces a global distance function $d:\mathcal{M}\times\mathcal{M}\to\mathbb{R}_+$ by setting $d(x,y)$ to be the infimum of all lengths of smooth curves joining $x$ to $y$ in $\mathcal{M}$, where the length $\ell$ of a curve $\gamma:[0,1]\to\mathcal{M}$ is defined as 
\begin{equation}
\ell(\gamma)=\int_{0}^1\sqrt{g_{\gamma(t)}(\gamma'(t),\gamma'(t))}dt.
\end{equation}

\paragraph{Riemannian manifold.} A smooth manifold equipped with a Riemannian metric is called a \textit{Riemannian manifold}. Subsequently, due to their metric properties, we will only consider such manifolds.

\paragraph{Geodesics.} A \textit{geodesic} (straight line) between two points $x,y\in\mathcal{M}$ is a smooth curve of minimal length joining $x$ to $y$ in $\mathcal{M}$. Geodesics define shortest paths on the manifold. They are a generalization of lines in the Euclidean space. 

\paragraph{Exponential map.} The \textit{exponential map} $\exp_x: T_x\mathcal{M}\to\mathcal{M}$ around $x$, when well-defined, maps a small perturbation of $x$ by a vector $v\in T_x\mathcal{M}$ to a point $\exp_x(v)\in\mathcal{M}$, such that $t\in[0,1]\mapsto\exp_x(tv)$ is a geodesic joining $x$ to $\exp_x(v)$. In Euclidean space, we simply have $\exp_x(v)=x+v$. The exponential map is important, for instance, when performing gradient-descent over parameters lying in a manifold~\citep{bonnabel2013stochastic}.

\paragraph{Conformality.} A metric $\tilde{g}$ on $\mathcal{M}$ is said to be \textit{conformal} to $g$ if it defines the same angles, \textit{i.e.} for all $x\in\mathcal{M}$ and $u,v\in T_x\mathcal{M}\setminus\lbrace 0\rbrace$, 
\begin{equation}
\frac{\tilde{g}_x(u,v)}{\sqrt{\tilde{g}_x(u,u)}\sqrt{\tilde{g}_x(v,v)}}=\frac{g_x(u,v)}{\sqrt{g_x(u,u)}\sqrt{g_x(v,v)}}.
\end{equation}
This is equivalent to the existence of a smooth function $\lambda:\mathcal{M}\to(0,\infty)$ such that $\tilde{g}_x=\lambda_x^2 g_x$, which is called the \textit{conformal factor} of $\tilde{g}$ (w.r.t. $g$).

\subsection{Hyperbolic geometry}\label{ssec:hyperbolic}

The hyperbolic space of dimension $n \geq 2$ is a fundamental object in Riemannian geometry. It is (up to isometry) uniquely characterized as a complete, simply connected Riemannian manifold with constant negative curvature \citep{cannon1997hyperbolic}. The other two model spaces of constant sectional curvature are the flat Euclidean space $\Re^n$ (zero curvature) and the hyper-sphere $\mathbb{S}^n$ (positive curvature).

The hyperbolic space has five models which are often insightful to work in. They are isometric to each other and conformal to the Euclidean space \citep{cannon1997hyperbolic,parkkonenhyperbolic}\footnote{\url{https://en.wikipedia.org/wiki/Hyperbolic_space}}. We prefer to work in the Poincar\'e ball model $\D^n$ for the same reasons as \citep{nickel2017poincar} and, additionally, because we can derive a closed form expression of geodesics and exponential map.
\paragraph{Poincar\'e metric tensor.} The Poincar\'e ball model $(\mathbb D^n, g^{\mathbb D})$ is defined by the manifold $\mathbb D^n = \{ x \in \mathbb \Re^n: \| x\| <1\}$ equipped with the following Riemannian metric
\begin{align}
g^\D_x = \lambda_x^2 g^E, \quad \text{where\ } \lambda_x := \frac 2 {1- \|x\|^2},
\label{eq:metric_tensor}
\end{align}
and $g^E$ is the Euclidean metric tensor with components $\I_n$ of the standard space $\Re^n$ with the usual Cartesian coordinates.

As the above model is a Riemannian manifold, its metric tensor is fundamental in order to uniquely define most of its geometric properties like distances, inner products (in tangent spaces), straight lines (geodesics), curve lengths or volume elements. In the Poincar\'e ball model, the Euclidean metric is changed by a simple scalar field, hence the model is \textit{conformal} (\textit{i.e.}~angle preserving), yet distorts distances.

\paragraph{Induced distance and norm.} It is known \citep{nickel2017poincar} that the induced distance between 2 points $x, y \in \D^n$ is given by
\begin{align}
d_{\D}(x,y) = \cosh^{-1}\left(1+ 2 \frac{\| x-y\|^2}{(1-\|x\|^2) \cdot (1-\|y\|^2)} \right) \,.
\end{align}
The Poincare norm is then defined as: 
\begin{align}
\|x\|_{\D} := d_{\D}(0,x) = 2 \tanh^{-1}(\|x\|) 
\label{eq:poincare_norm}
\end{align}

\paragraph{Geodesics and exponential map.} We derive parametric expressions of unit-speed geodesics and exponential map in the Poincar\'e ball. Geodesics in $\mathbb D^n$  are all intersections of the Euclidean unit ball $\D^n$ with (degenerated) Euclidean circles orthogonal to the unit sphere $\partial \mathbb D^n$ (equations are derived below). We know from the Hopf-Rinow theorem that the hyperbolic space is complete as a metric space. This guarantees that $\D^n$ is geodesically complete. Thus, the exponential map is defined for each point $x \in \D^n$ and any $v \in \R^n (= T_x\D^n)$. To derive its closed form expression, we first prove the following.
\begin{theorem} (Unit-speed geodesics)
\label{thm:geodesic_unit_speed}
Let $x \in \D^n$ and $v \in T_x\D^n (= \R^n)$ such that $g_x^\D(v,v) = 1$. The unit-speed geodesic $\gamma_{x,v}: \Re_+ \rightarrow \D^n$ with $\gamma_{x,v}(0) = x$ and $\dot{\gamma}_{x,v}(0) = v$ is given by
\begin{align}
\gamma_{x,v}(t) = \frac{\left( \lambda_x \cosh(t) + \lambda_x^2 \inp{x}{v} \sinh(t) \right) x + \lambda_x \sinh(t) v}{1 + (\lambda_x - 1) \cosh(t) + \lambda_x^2 \inp{x}{v} \sinh(t) }
\label{eq:geodesic_unit_speed}
\end{align}
\end{theorem}
\begin{proof}
See appendix~\ref{sec:appendix_thm1_proof}.
\end{proof}

\begin{corollary} (Exponential map)
\label{cor:exp_map} 
The exponential map at a point $x \in \D^n$, namely $\exp_x: T_x\D^n \to \D^n$, is given by
\begin{multline}
\exp_x(v) =\\ \frac{\lambda_x \left( \cosh(\lambda_x \|v\|) + \inp{x}{\frac{v}{\|v\|}} \sinh(\lambda_x \|v\|) \right)}{1 + (\lambda_x - 1) \cosh(\lambda_x \|v\|) + \lambda_x \inp{x}{ \frac{v}{\|v\|}} \sinh(\lambda_x \|v\|) }x +\\
 \frac{\frac{1}{\|v\|} \sinh(\lambda_x \|v\|)}{1 + (\lambda_x - 1) \cosh(\lambda_x \|v\|) + \lambda_x \inp{x}{ \frac{v}{\|v\|}} \sinh(\lambda_x \|v\|) }v
\end{multline}
\end{corollary}
\begin{proof}
See appendix~\ref{proof:exp_map}.
\end{proof}

We also derive the following fact (useful for future proofs).
\begin{corollary}
Given any arbitrary geodesic in $\D^n$, all its points are coplanar with the origin $O$.
\label{cor:coplanar}
\end{corollary}
\begin{proof}
See appendix~\ref{proof:coplanar}.
\end{proof}

\paragraph{Angles in hyperbolic space.}  It is natural to extend the Euclidean notion of an angle to any geodesically complete Riemannian manifold. For any points A, B, C on such a manifold, the angle $\angle ABC$ is the angle between the initial tangent vectors of the geodesics connecting B with A, and B with C, respectively. In the Poincar\'e ball, the angle between two tangent vectors $u,v\in T_x\D^n$ is given by 
\begin{equation}
\cos(\angle(u,v))=\frac{g^{\D}_x(u,v)}{\sqrt{g^{\D}_x(u,u)}\sqrt{g^{\D}_x(v,v)}} = \frac{\inp{u}{v}}{\|u\| \|v\|}
\label{eq:angle-geodesics}
\end{equation}
The second equality happens since $g^{\D}$ is conformal to $g^E$. 

\paragraph{Hyperbolic trigonometry.} The notion of angles and geodesics allow  definition of the notion of a triangle in the Poincar\'e ball. Then, the classic theorems in Euclidean geometry have hyperbolic formulations~\citep{parkkonenhyperbolic}. In the next section, we will use the following theorems.

Let $A, B, C \in \D^n$. Denote by $\angle B := \angle ABC$ and by $c = d_\D(B,A)$ the length of the hyperbolic segment BA (and others). Then, the hyperbolic laws of cosines and sines hold respectively
\begin{align}
\cos (\angle B) = \frac{\cosh(a) \cosh(c) - \cosh(b)}{\sinh(a) \sinh(c)} \\
\frac{\sin (\angle A)}{\sinh(a)} = \frac{\sin (\angle B)}{\sinh(b)}  = \frac{\sin (\angle C)}{\sinh(c)}
\label{eq:sine_law}
\end{align}

\paragraph{Embedding trees in hyperbolic vs Euclidean space.}\ \ Finally, we briefly explain why hyperbolic spaces are better suited than Euclidean spaces for embedding trees. However, note that our method is applicable to any DAG.

\citep{gromov1987hyperbolic} introduces a notion of $\delta$-hyperbolicity in order to characterize how `hyperbolic' a metric space is. For instance, the Euclidean space $\Re^n$ for $n\geq 2$ is not $\delta$-hyperbolic for any $\delta\geq 0$, while the Poincar\'e ball $\D^n$ is $\log(1+\sqrt{2})$-hyperbolic. This is formalized in
the following theorem\footnote{\url{https://en.wikipedia.org/wiki/Hyperbolic_metric_space}} (section 6.2 of \citep{gromov1987hyperbolic}, proposition 6.7 of \citep{bowditch2006course}): 

\textit{Theorem}: For any $\delta > 0$, any $\delta$-hyperbolic metric space $(X,d_X)$ and any set of points $x_1,...,x_n\in X$, there exists a finite weighted tree $(T,d_T)$ and an embedding $f:T\to X$ such that for all $i,j$,
\begin{equation}
\vert d_T(f^{-1}(x_i),f^{-1}(x_j))- d_X(x_i,x_j)\vert=\mathcal{O}(\delta\log(n)).
\end{equation}

Conversely, any tree can be embedded with arbitrary low distortion into the Poincar\'e disk (with only 2 dimensions), whereas this is not true for Euclidean spaces even when an unbounded number of dimensions is allowed~\citep{sarkar2011low,de2018representation}.

The difficulty in embedding trees having a branching factor at least $2$ in a quasi-isometric manner comes from the fact that they have an exponentially increasing number of nodes with depth. The exponential volume growth of hyperbolic metric spaces confers them enough capacity to embed trees quasi-isometrically, unlike the Euclidean space.

\section{Entailment Cones in the Poincar\'e Ball}\label{sec:theory}

In this section, we define ``entailment'' cones that will be used to embed hierarchical structures in the Poincar\'e ball. They generalize and improve over the idea of order embeddings \citep{vendrov2015order}.

\begin{figure}[!htp]
 \centering
  \includegraphics[width=0.4\textwidth]{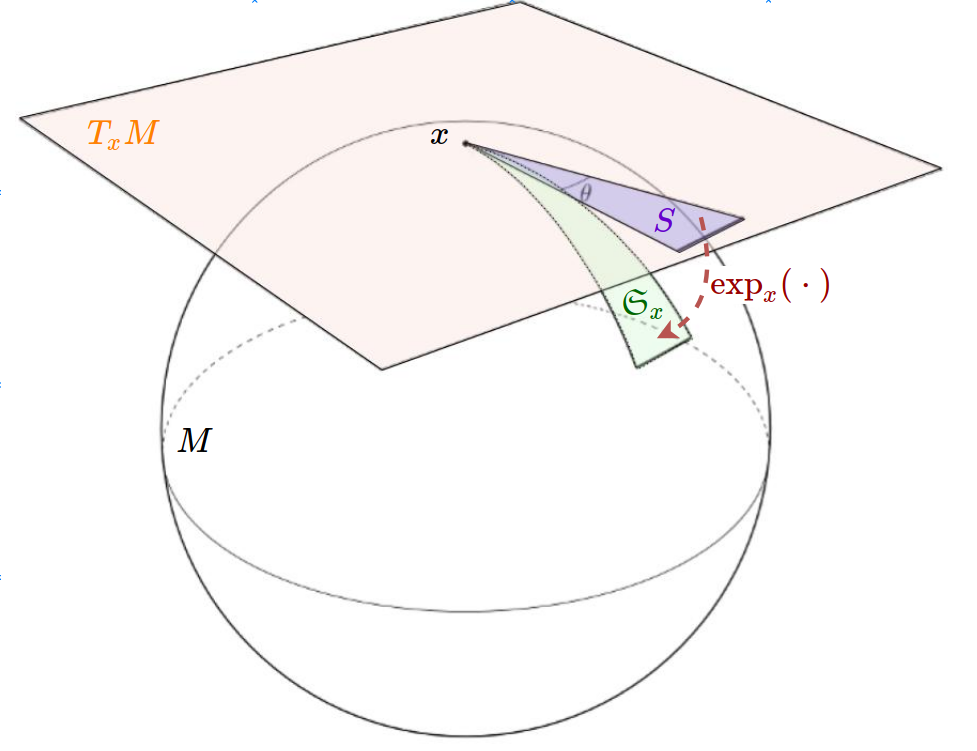}
\caption{Convex cones in a complete Riemannian manifold. }
\label{fig:convex_cones}
\end{figure}

\paragraph{Convex cones in a complete Riemannian manifold.}
We are interested in generalizing the notion of a convex cone to any geodesically complete Riemannian manifold $\mathcal{M}$ (such as hyperbolic models). In a vector space, a convex cone $S$ (at the origin) is a set that is closed under non-negative linear combinations
\begin{align}
v_1, v_2 \in S  \quad \Longrightarrow \alpha v_1 + \beta v_2 \in S \quad  (\forall \alpha, \beta \geq 0) \,.
\end{align}
The key idea for generalizing this concept is to make use of the exponential map at a point $x \in\mathcal{M}$.
\begin{align}
\exp_x: T_x\mathcal{M}  \to \mathcal{M},\quad T_x\mathcal{M}  =\; \text{tangent space at $x$}
\end{align}
We can now take any cone in the tangent space $S \subseteq T_x\mathcal{M}$ at a fixed point $x$ and map it into a set $\frak S_x \subset\mathcal{M}$, which we call the $S$-cone at $x$, via
\begin{align}
\frak S_x := \exp_x \left( S \right), \quad S \subseteq T_x\mathcal{M} \,.
\label{eq:riem_cone}
\end{align}
Note that, in the above definition, we desire that the exponential map be injective. We already know that it is a local diffeomorphism. Thus, we restrict the tangent space in Eq.~\ref{eq:riem_cone} to the ball $\mathcal{B}^n(O,r)$, where $r$ is the injectivity radius of $M$ at $x$. Note that for hyperbolic space models the injectivity radius of the tangent space at any point is infinite, thus no restriction is needed.

\paragraph{Angular cones in the Poincar\'e ball.}
We are interested in special types of cones in $\D^n$ that can extend in all space directions. We want to avoid heavy cone intersections and to have capacity that scales exponentially with the space dimension. To achieve this, we want the definition of cones to exhibit the following four intuitive properties detailed below. Subsequently, solely based on these necessary conditions, we formally prove that the optimal cones in the Poincar\'e ball have a closed form expression.

{\bf 1) Axial symmetry.} For any $x \in \D^n \setminus \{0\}$, we require circular symmetry with respect to a central axis of the cone $\frak S_x$. We define this axis to be the spoke through $x$ from $x$:
\begin{align}
A_x := \{ x' \in \D^n : x' = \alpha x ,  \ \frac{1}{\|x\|} > \alpha \geq 1 \}
\label{eq:spoke}
\end{align}

Then, we fix any tangent vector with the same direction as $x$, e.g. $\bar{x} = \exp_x^{-1}\left(\frac{1 + \|x\|}{2 \|x\|}x\right) \in T_x\D^n$. One can verify using Corollary~\ref{cor:exp_map} that $\bar{x}$ generates the axis-oriented geodesic as:

\begin{align}
A_x = \exp_x\left( \{ y \in \R^n : y = \alpha \bar{x} , \ \alpha > 0 \} \right).
\end{align}

We next define the angle $\angle(v, \bar{x})$ for any tangent vector $v \in T_x\D^n$ as in Eq.~\ref{eq:angle-geodesics}. Then, the axial symmetry property is satisfied if we define the angular cone at $x$ to have a non-negative aperture $2\psi(x) \ge 0$ as follows:
\begin{align}
& S_x^{\psi(x)} := \{ v \in T_x\D^n: \angle(v, \bar{x}) \le \psi(x) \} \\ \nonumber
& \frak S^{\psi(x)}_x := \exp_x(S_x^{\psi(x)}).
\end{align}

We further define the conic border (face):
\begin{align}
\partial S^\psi := \{ v: \angle(v, \bar{x}) = \psi(x) \}, \quad \partial \frak S^\psi_x := \exp_x(\partial S_x^\psi).
\end{align}

{\bf 2) Rotation invariance.} We want the definition of cones $\frak S^{\psi(x)}_x$ to be independent of the angular coordinate of the apex $x$, \textit{i.e.} to only depend on the (Euclidean) norm of $x$:
\begin{align}
\psi(x) = \psi(x')  \quad (\forall x,x' \in \D^n\setminus \{0\},\ \mathrm{s.t.\ }\Vert x\Vert=\Vert x'\Vert).
\label{eq:rotation_invariant}
\end{align}

This implies that there exists $\tilde{\psi}:(0,1)\to[0,\pi)$ s. t. for all $x\in\D^n\setminus \{0\}$ we have $\psi(x)=\tilde{\psi}(\Vert x\Vert)$.

{\bf 3) Continuous cone aperture functions}. We require the aperture $\psi$ of our cones to be a continuous function. Using Eq.~\ref{eq:rotation_invariant}, it is equivalent to the continuity of $\tilde{\psi}$. This requirement seems reasonable and will be helpful in order to prove uniqueness of the optimal entailment cones. When optimization-based training is employed, it is also necessary that this function be differentiable. Surprisingly, we will show below that the optimal functions $\tilde{\psi}$ are actually smooth, even when only requiring continuity.


{\bf 4) Transitivity of nested angular cones.} We want cones to determine a partial order in the embedding space. The difficult property is transitivity. We are interested in defining a cone width function $\psi(x)$ such that the resulting angular cones satisfy the \textit{transitivity} property of partial order relations, \textit{i.e.} they form a nested structure as follows
\begin{align}
\forall x,x' \in \D^n\setminus \{0\}: \quad x' \in \frak S^{\psi(x)}_x  \; \Longrightarrow \; \frak S^{\psi(x')}_{x'}
\subseteq \frak S^{\psi(x)}_x.
\label{eq:cone_transitivity}
\end{align}

\paragraph{Closed form expression of the optimal $\psi$.} We now analyze the implications of the above necessary properties. Surprisingly, the optimal form of the function $\psi$ admits an interesting closed-form expression. We will see below that mathematically $\psi$ cannot be defined on the entire open ball $\D^n$. Towards these goals, we first prove the following.

\begin{lemma}
\label{lemma:acute_angle}
If transitivity holds, then 
\begin{align}
\forall x \in \text{Dom}(\psi): \quad \psi(x) \leq \frac{\pi}{2}.
\end{align}
\end{lemma}
\begin{proof}
See appendix~\ref{proof:lemma_acute_angle}.
\end{proof}

Note that so far we removed the origin $0$ of $\D^n$ from our definitions. However, the above surprising lemma implies that we cannot define a useful cone at the origin. To see this, we first note that the origin should ``entail" the entire space $\D^n$, i.e. $\frak S_0 = \D^n$. Second, similar with property 3, we desire the cone at 0 be a continuous deformation of the cones of any sequence of points $(x_n)_{n \geq 0}$ in $\D^n\setminus \{0\}$ that converges to 0. Formally, $\lim_{n \rightarrow \infty} \frak S_{x_n} = \frak S_0$ when $\lim_{n \rightarrow \infty} x_n = 0$. However, this is impossible because Lemma~\ref{lemma:acute_angle} implies that the cone at each point $x_n$ can only cover at most half of $\D^n$. We further prove the following:
\begin{theorem}
\label{thm:sin_norm_product}
If transitivity holds, then the function
\begin{align}
h : (0,1) \cap \text{Dom}(\tilde{\psi}) \rightarrow \R_+, \quad h(r):= \frac{r}{1 - r^2} \sin(\tilde{\psi}(r)),
\end{align}
is non-increasing.
\end{theorem}
\begin{proof}
See appendix~\ref{proof:sin_norm_product}.
\end{proof}

\begin{figure}
 \centering
  \includegraphics[width=0.43\linewidth, height=0.43\linewidth]{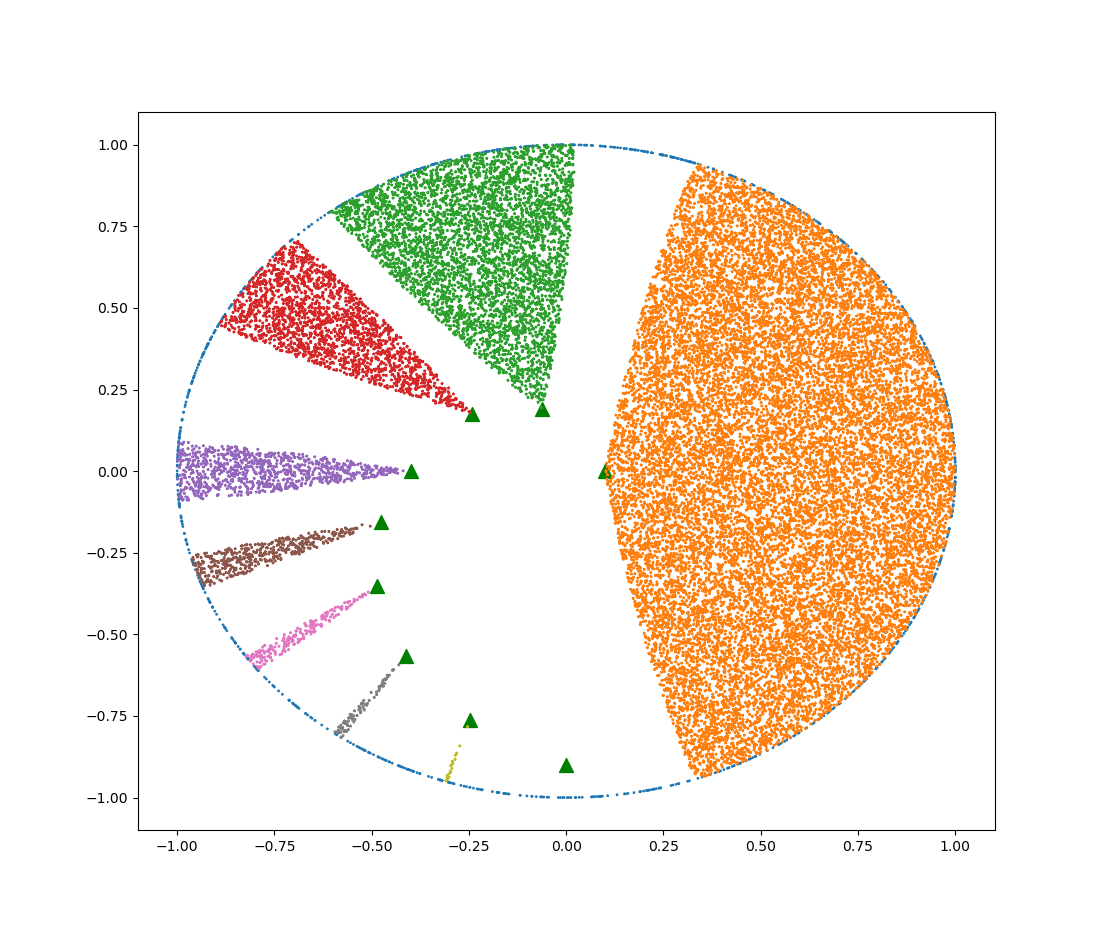}
  \includegraphics[width=0.56\linewidth, height=0.43\linewidth]{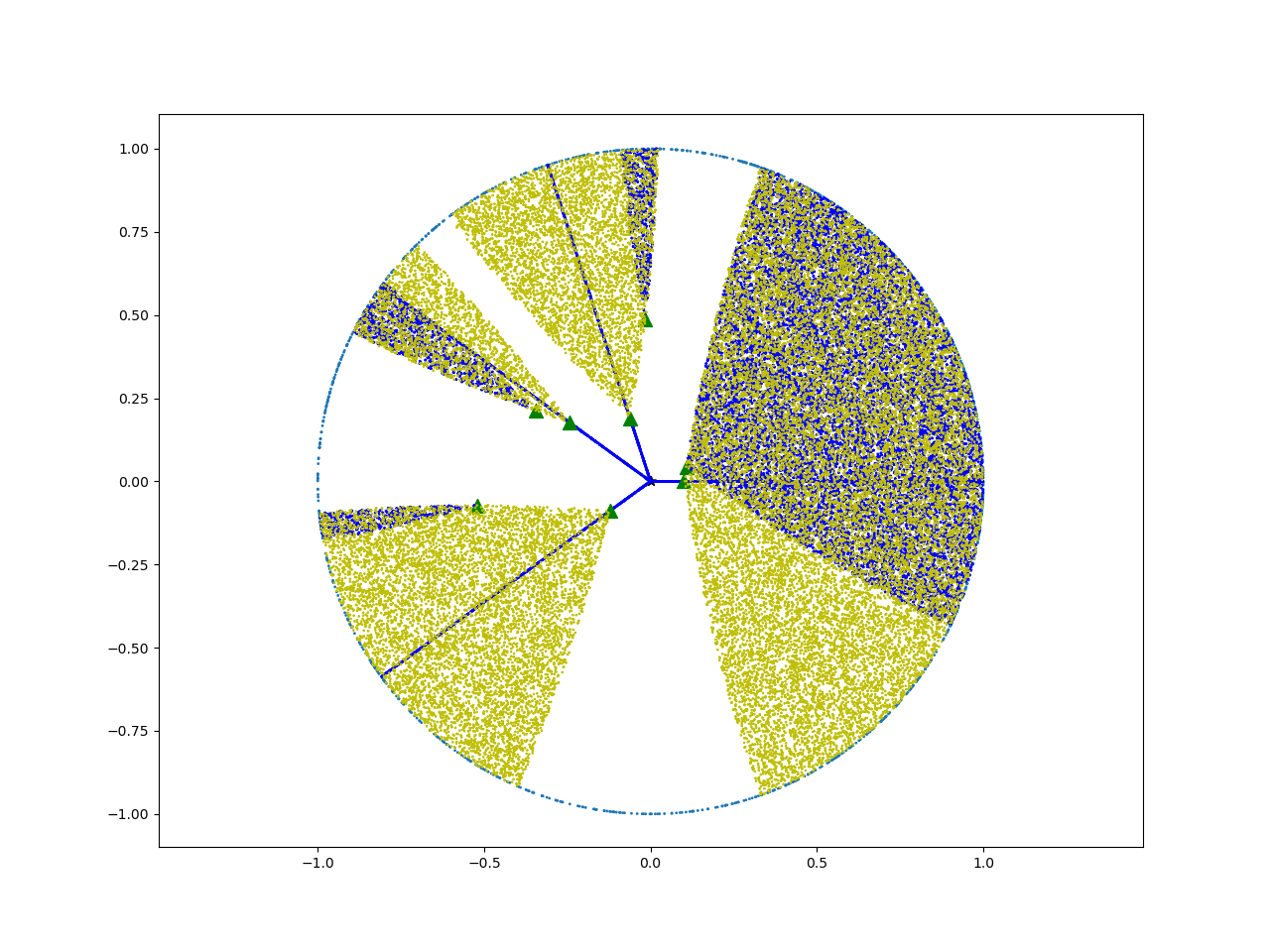}
\caption{Poincar\'e angular cones satisfying Eq.~\ref{eq:cones_closed_form} for K = 0.1. Left: examples of cones for points with Euclidean norm varying from 0.1 to 0.9. Right: transitivity for various points on the border of their parent cones. }
\label{fig:hyp_angular}
\end{figure}

The above theorem implies that a non-zero $\tilde{\psi}$ cannot be defined on the entire $(0,1)$ because $\lim_{r \rightarrow 0} h(r) = 0$, for any function $\tilde{\psi}$. As a consequence, we are forced to restrict $\text{Dom}(\tilde{\psi})$ to some $[\epsilon, 1)$, \textit{i.e.} to leave the open ball $\mathcal{B}^n(O,\epsilon)$ outside of the domain of $\psi$. Then, theorem~\ref{thm:sin_norm_product} implies that
\begin{align}
\forall r \in [\epsilon, 1) : \quad \sin(\tilde{\psi}(r)) \frac{r}{1 - r^2}  \leq \sin(\tilde{\psi}(\epsilon)) \frac{\epsilon}{1 - \epsilon^2}.
\end{align}
Since we are interested in cones with an aperture as large as possible (to maximize model capacity), it is natural to set all terms $h(r)$ equal to $K := h(\epsilon)$, \textit{i.e.} to make $h$ constant:
\begin{align}
\forall r \in [\epsilon, 1) : \quad \sin(\tilde{\psi}(r)) \frac{r}{1 - r^2}  = K,
\end{align}
which gives both a restriction on $\epsilon$ (in terms of $K$):
\begin{align}
K \leq \frac{\epsilon}{1 - \epsilon^2} \quad \iff \quad \epsilon \in \left[ \frac{2K}{1 + \sqrt{1 + 4K^2}}, 1  \right),
\label{eq:closed_form_eps}
\end{align}
as well as a closed form expression for $\psi$
\begin{align}
\psi:\ & \D^n \setminus \mathcal{B}^n(O,\epsilon) \rightarrow (0,\pi/2) \nonumber \\ 
& x \mapsto \arcsin (  K (1 - \|x\|^2)/\|x\| ),
\label{eq:cones_closed_form}
\end{align}
which is also a sufficient condition for transitivity to hold:

\begin{theorem}
If $\psi$ is defined as in Eqs.\ref{eq:closed_form_eps}-\ref{eq:cones_closed_form}, then transitivity holds.
\end{theorem}
The above theorem has a proof similar to that of Thm.~\ref{thm:sin_norm_product}.

%
So far, we have obtained a closed form expression for hyperbolic entailment cones. However, we still need to understand how they can be used during embedding learning. For this goal, we derive an equivalent (and more practical) definition of the cone $\frak S^{\psi(x)}_x$:
\begin{theorem}
\label{thm:equiv_def}
For any $x,y\in\D^n\setminus \mathcal{B}^n(O,\epsilon)$, we denote the angle between the half-lines $(xy$ and $(0x$ as
\begin{align}
\Xi(x,y):=\pi - \angle Oxy,
\end{align}
Then, this angle equals 
\begin{align}
\arccos \left( \frac{\inp{x}{y} (1 + \|x\|^2) - \|x\|^2 (1 + \|y\|^2)}{\|x\| \cdot \|x-y\| \sqrt{1 + \|x\|^2 \|y\|^2 - 2 \inp{x}{y}}} \right),
\label{eq:angle_oxy}
\end{align}
Moreover, we have the following equivalent expression of the Poincar\'e entailment cones satisfying Eq.~\ref{eq:cones_closed_form}:
\begin{align}
\frak S^{\psi(x)}_x = \Set{ y \in \D^n | \Xi(x,y) \leq  \arcsin \left(  K \frac{1 - \|x\|^2}{\|x\|} \right) }.
\end{align}
\end{theorem}
\begin{proof}
See appendix~\ref{proof:equiv_def}.
\end{proof}

Examples of 2-dimensional Poincar\'e cones corresponding to apex points located at different radii from the origin are shown in Figure~\ref{fig:hyp_angular}. This figure also shows that transitivity is satisfied for some points on the border of the hypercones.


\paragraph{Euclidean entailment cones.} One can easily adapt the above proofs to derive entailment cones in the Euclidean space $(\Re^n,g^E)$. The only adaptations are: i) replace the hyperbolic cosine law by usual Euclidean cosine law, ii) geodesics are straight lines, and iii) the exponential map is given by $\exp_x(v)=x+v$. Thus, one similarly obtains that $h(r)=r\sin(\psi(r))$ is non-decreasing, the optimal values of $\psi$ are obtained for constant $h$ being equal to $K \leq \varepsilon$ and
\begin{align}
\frak{S}^{\psi(x)}_x=\lbrace y\in\Re^n\mid \Xi(x,y)\leq \psi(x)\rbrace,
\end{align} 
where $\Xi(x,y)$ now becomes
\begin{align}
\Xi(x,y)=\arccos\left(\dfrac{\Vert y\Vert^2-\Vert x\Vert^2-\Vert x-y\Vert^2}{2\Vert x\Vert\cdot \Vert x-y\Vert}\right),
\label{eq:psi_eucl}
\end{align}
for all $x,y\in\Re^n\setminus\mathcal{B}(O,\varepsilon)$.
From a learning perspective, there is no need to be concerned about the Riemannian optimization described in Section~\ref{sec:riem_opt}, as the usual Euclidean gradient-step is used in this case.

\begin{figure*}[!htp]
 \centering
  \includegraphics[width=0.22\linewidth, height=0.22\linewidth]{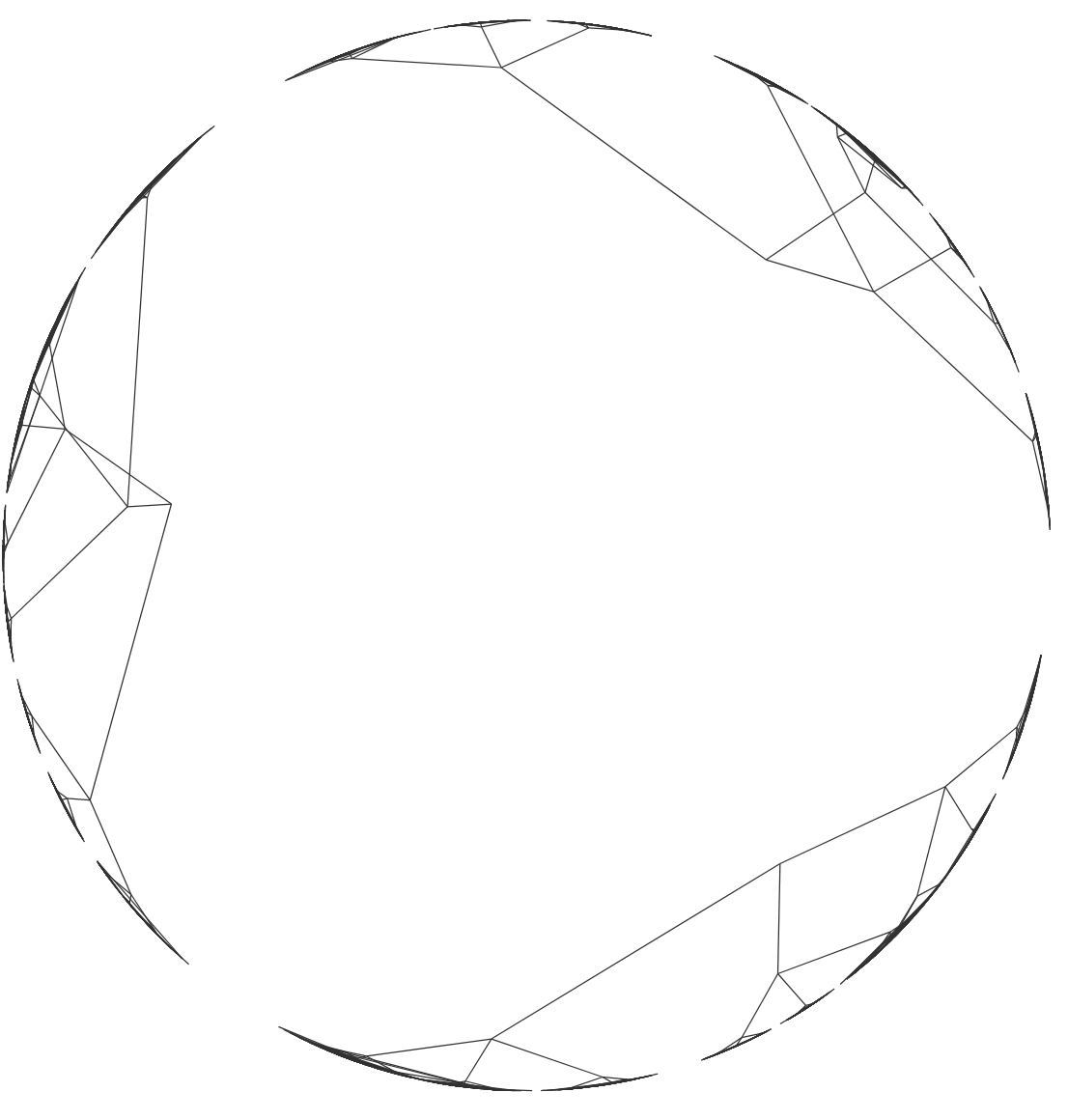}
  \includegraphics[width=0.22\linewidth, height=0.22\linewidth]{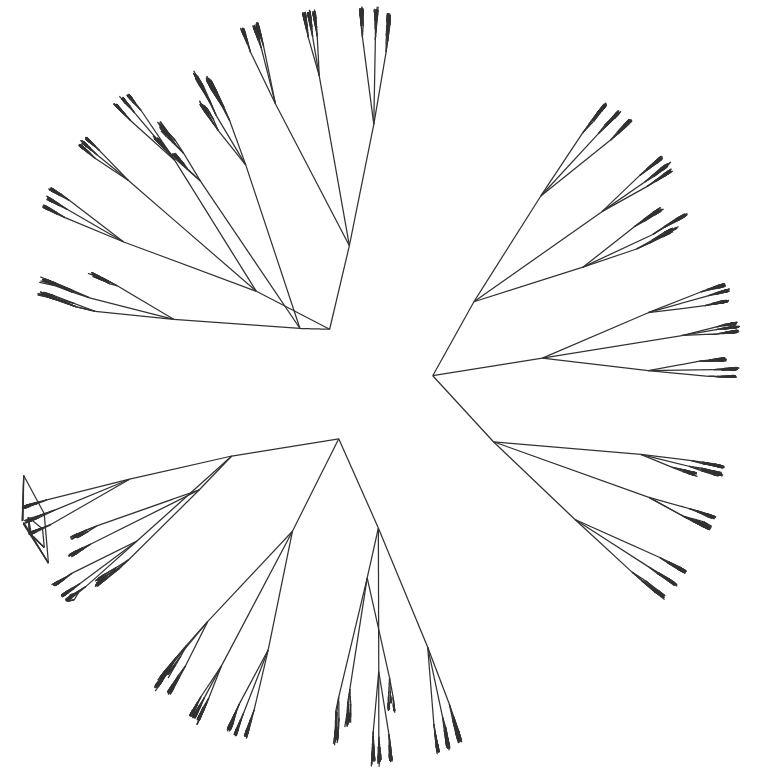}
  \hspace{0.9cm}
  \includegraphics[width=0.24\linewidth, height=0.22\linewidth]{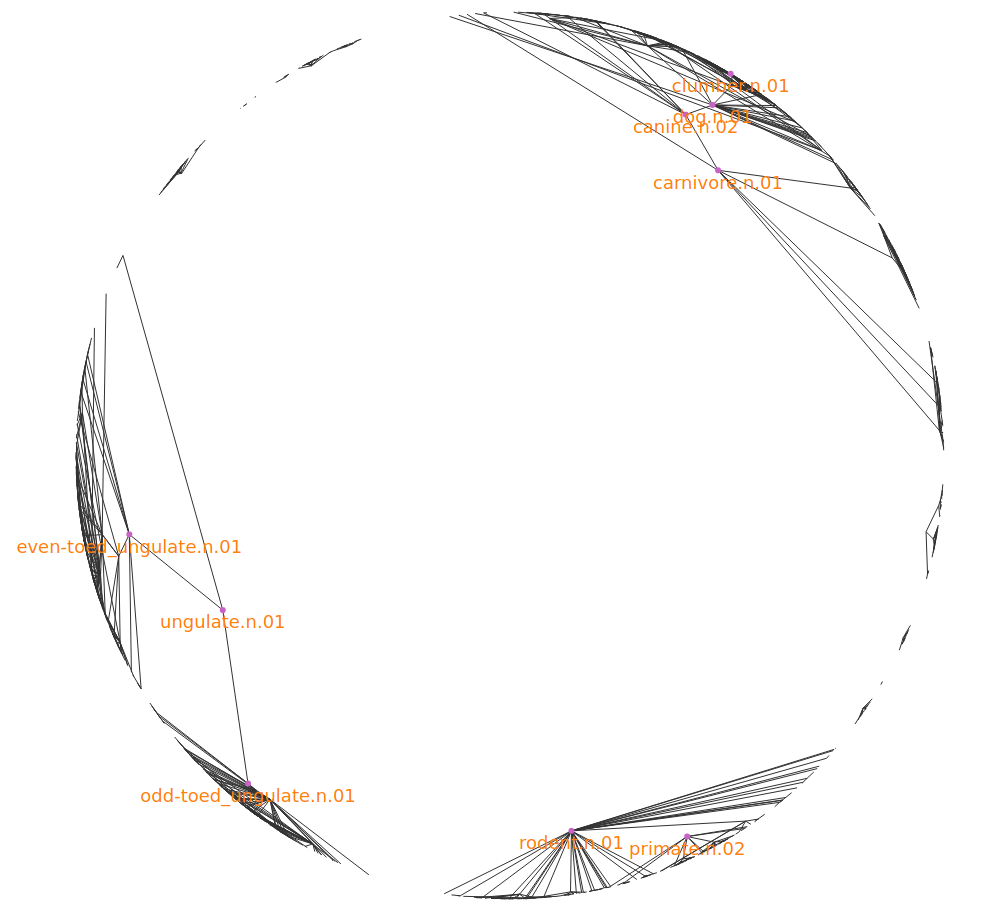}
  \includegraphics[width=0.22\linewidth, height=0.22\linewidth]{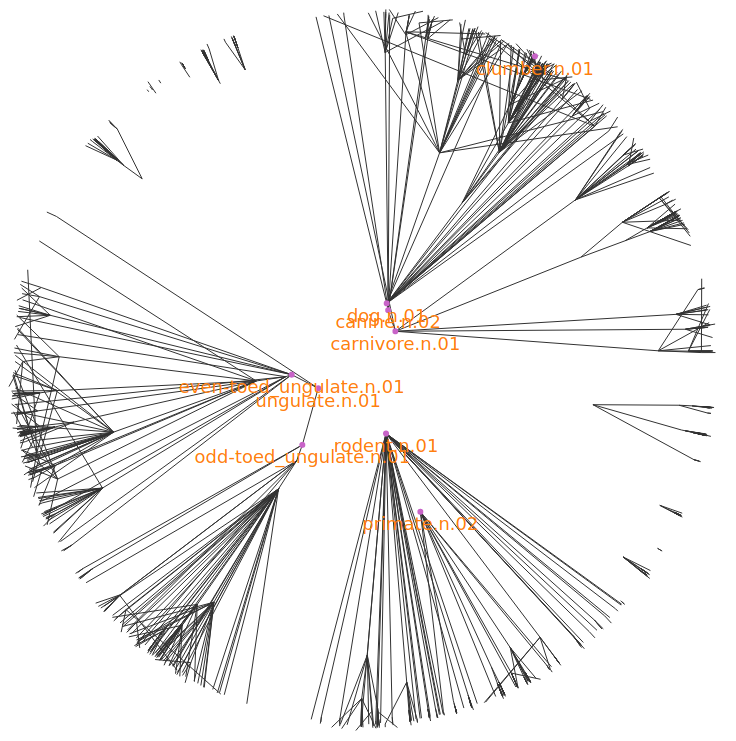}
\caption{Two dimensional embeddings of two datasets: a toy uniform tree of depth 7 and branching factor 3, with root removed (left); the mammal subtree of WordNet with 4230 relations, 1165 nodes and top 2 nodes removed (right).  \citep{nickel2017poincar} (each left side) has most of the nodes and edges collapsed on the space border, while our hyperbolic cones (each right side) nicely reveal the data structure.}
\label{fig:mammals}
\end{figure*}

\section{Learning with entailment cones}\label{sec:learning}
We now describe how embedding learning is performed.

\subsection{Max-margin training on angles}

We learn hierarchical word embeddings from a dataset $\mathcal{X}$ of entailment relations $(u,v)\in\mathcal{X}$, also called hypernym links, defining that $u$ entails $v$, or, equivalently, that $v$ is a subconcept of $u$\footnote{We prefer this notation over the one in \citep{nickel2017poincar}}. 

We choose to model the embedding entailment relation $(u,v)$ as $v$ belonging to the entailment cone $\frak S^{\psi(u)}_u$. 

Our model is trained with a max-margin loss function similar to the one in \citep{vendrov2015order}:
\begin{align}
\mathcal{L}=\sum_{(u,v)\in P} E(u,v)+\sum_{(u',v')\in N}\max(0,\gamma-E(u',v')),
\end{align}
for some margin $\gamma>0$, where $P$ and $N$ define samples of positive and negative edges respectively. The energy $E(u,v)$ measures the penalty of a wrongly classified pair $(u,v)$, which in our case measures how far is point $v$ from belonging to $\frak S^{\psi(u)}_u$ expressed as the smallest angle of a rotation of center $u$ bringing $v$ into $\frak S^{\psi(u)}_u$:
\begin{align}
E(u,v):=\max(0,\Xi(u,v)-\psi(u)),
\label{eq:energy_hyp}
\end{align}
where $\Xi(u,v)$ is defined in Eqs.~\ref{eq:angle_oxy} and~\ref{eq:psi_eucl}. Note that \citep{vendrov2015order} use $\Vert\max(0,v-u)\Vert^2$. This loss function encourages positive samples to satisfy $E(u,v)=0$ and negative ones to satisfy $E(u,v) \geq \gamma$. The same loss is used both in the hyperbolic and Euclidean cases.

\subsection{Full Riemannian optimization}
\label{sec:riem_opt}
As the parameters of the model live in the hyperbolic space, the back-propagated gradient is a Riemannian gradient. Indeed, if $u$ is in the Poincar\'e ball, and if we compute the usual (Euclidean) gradient $\nabla_u\mathcal{L}$ of our loss, then 
\begin{align}
u\leftarrow u-\eta\nabla_u\mathcal{L}
\end{align} 
makes no sense as an operation in the Poincar\'e ball, since the substraction operation is not defined in this manifold. Instead, one should compute the Riemannian gradient $\nabla^R_u\mathcal{L}$ indicating a direction in the tangent space $T_u\D^n$, and should move $u$ along the corresponding geodesic in $\D^n$~\citep{bonnabel2013stochastic}:
\begin{align}
u\leftarrow\exp_u(-\eta\nabla^R_u\mathcal{L}),
\end{align}
where the Riemannian gradient is obtained by rescaling the Euclidean gradient by the inverse of the metric tensor. As our metric is conformal, \textit{i.e.} $g^\D=\lambda^2 g^E$ where $g^E=\I$ is the Euclidean metric (see Eq \ref{eq:metric_tensor}), this leads to a simple formulation
\begin{align}
\nabla^R_u\mathcal{L}=(1/\lambda_u)^2\nabla_u\mathcal{L}.
\end{align}
Previous work \citep{nickel2017poincar} optimizing word embeddings in the Poincar\'e ball used the retraction map $\mathcal{R}_x(v):=x+v$ as a first order approximation of $\exp_x(v)$. Note that since we derived a closed-form expression of the exponential map in the Poincar\'e ball (Corollary~\ref{cor:exp_map}), we are able to perform full Riemannian optimization in this model of the hyperbolic space.

\section{Experiments}\label{sec:exp}

\begin{table*}[t]
\begin{center}
\begin{small}
\begin{sc}
\begin{tabular}{| c || c | c | c | c || c | c | c | c |}
\toprule
                      &         \multicolumn{4}{c||}{Embedding Dimension = 5}  & \multicolumn{4}{c|}{Embedding Dimension = 10} \\  \hline
                      &         \multicolumn{8}{c|}{Percentage of Transitive Closure (Non-basic) Edges in Training }  \\ 

                      &   0\%  &  10\%  &   25\% &   50\% &    0\% &   10\% &   25\% &   50\% \\
\midrule
Simple Euclidean Emb  & 26.8\% & 71.3\% & 73.8\% & 72.8\% & 29.4\% & 75.4\% & 78.4\% & 78.1\% \\ \hline
Poincar\'e Emb        & 29.4\% & 70.2\% & 78.2\% & 83.6\% & 28.9\% & 71.4\% & 82.0\% & 85.3\% \\ \hline
Order Emb             & \textbf{34.4}\% & 70.2\% & 75.9\% & 81.7\% & \textbf{43.0}\% & 69.7\% & 79.4\% & 84.1\% \\ \hline
Our Euclidean Cones   & 28.5\% & 69.7\% & 75.0\% & 77.4\% & 31.3\% & 81.5\% & 84.5\% & 81.6\% \\ \hline
Our Hyperbolic Cones  & 29.2\% & \textbf{80.1}\% & \textbf{86.0}\% & \textbf{92.8}\% & 32.2\% & \textbf{85.9}\% & \textbf{91.0}\% & \textbf{94.4}\% \\ \hline

\end{tabular}
\end{sc}
\end{small}
\end{center}
\vskip -0.1in
\caption{Test F1 results for various models. Simple Euclidean Emb and Poincar\'e Emb are the Euclidean and hyperbolic methods proposed by~\citep{nickel2017poincar}, Order Emb is proposed by~\citep{vendrov2015order}. }
\label{tab:results}
\end{table*}

We evaluate the representational and generalization power of hyperbolic entailment cones and of other baselines using data that exhibits a latent hierarchical structure. We follow previous work \citep{nickel2017poincar,vendrov2015order} and use the full transitive closure of the WordNet noun hierarchy~\citep{miller1990introduction}. Our binary classification task is link prediction for unseen edges in this directed acyclic graph.

\paragraph{Dataset splitting. Train and evaluation settings. }
We remove the tree root since it carries little information and only has trivial edges to predict. Note that this implies that we co-embed the resulting subgraphs together to prevent overlapping embeddings (see smaller examples in Figure~\ref{fig:mammals}). The remaining WordNet dataset contains 82,114 nodes and 661,127 edges in the full transitive closure. We split it into train - validation - test sets as follows. We first compute the transitive reduction\footnote{\url{https://en.wikipedia.org/wiki/Transitive_reduction}} of this directed acyclic graph, \textit{i.e.} \textit{``basic" edges} that form the minimal edge set for which the original transitive closure can be fully recovered. These edges are hard to predict, so we will always include them in the training set. The remaining \textit{``non-basic" edges} (578,477) are split into validation (5\%), test (5\%) and train (fraction of the rest).

 We augment both the validation and the test parts with sets of negative pairs as follows: for each true (positive) edge $(u,v)$, we randomly sample five $(u',v)$ and five $(u,v')$ negative corrupted pairs that are not edges in the full transitive closure. These are then added to the respective negative set. Thus, ten times as many negative pairs as positive pairs are used. They are used to compute standard classification metrics associated with these datasets: precision, recall, F1. For the training set, negative pairs are dynamically generated as explained below.

We make the task harder in order to understand the generalization ability of various models when differing amounts of transitive closure edges are available during training. We generate four training sets that include 0\%, 10\%, 25\%, or 50\% of the non-basic edges, selected randomly. We then train separate models using each of these four sets after being augmented with the basic edges.

\paragraph{Baselines.}
We compare against the strong hierarchical embedding methods of \textit{Order embeddings}~\cite{vendrov2015order} and \textit{Poincar\'e embeddings}~\citep{nickel2017poincar}. Additionally, we also use \textit{Simple Euclidean embeddings}, \textit{i.e.} the Euclidean version of the method presented in~\citep{nickel2017poincar} (one of their baselines). We note that Poincar\'e and Simple Euclidean embeddings were trained using a symmetric distance function, and thus cannot be directly used to evaluate asymmetric entailment relations. Thus, for these baselines we use the heuristic scoring function proposed in~\citep{nickel2017poincar}:
\begin{align}
score(u,v) = (1 + \alpha (\|u\| - \|v\|)) d(u,v)
\end{align}
and tune the parameter $\alpha$ on the validation set. For all the other methods (our proposed cones and order embeddings), we use the energy penalty $E(u,v)$, e.g. Eq.~\ref{eq:energy_hyp} for hyperbolic cones. This scoring function is then used at test time for binary classification as follows: if it is lower than a threshold, we predict an edge; otherwise, we predict a non-edge. The optimal threshold is chosen to achieve maximum F1 on the validation set by passing over the sorted  array of scores of positive and negative validation pairs.

\paragraph{Training details.} 

For all methods except Order embeddings, we observe that initialization is very important. Being able to properly disentangle embeddings from different subparts of the graph in the initial learning stage is essential in order to train qualitative models. We conjecture that initialization is hard because these models are trained to minimize highly non-convex loss functions. In practice, we obtain our best results when initializing the embeddings corresponding to the hyperbolic cones using the Poincar\'e embeddings pre-trained for 100 epochs. The embeddings for the Euclidean cones are initialized using Simple Euclidean embeddings pre-trained also for 100 epochs. For the Simple Euclidean embeddings and Poincar\'e embeddings, we find the burn-in strategy of~\citep{nickel2017poincar} to be essential for a good initial disentanglement. We also observe that the Poincar\'e embeddings are heavily collapsed to the unit ball border (as also pictured in Fig.~\ref{fig:mammals}) and so we rescale them by a factor of 0.7 before starting the training of the hyperbolic cones. 

Each model is trained for 200 epochs after the initialization stage, except for order embeddings which were trained for 500 epochs.  During training, 10 negative edges are generated per positive edge by randomly corrupting one of its end points. We use batch size of 10 for all models. For both cone models we use a margin of $\gamma=0.01$.


All Euclidean models and baselines are trained using stochastic gradient descent. For the hyperbolic models, we do not find significant empirical improvements when using full Riemannian optimization instead of approximating it with a retraction map as done in ~\citep{nickel2017poincar}. We thus use the retraction approximation since it is faster. For the cone models, we always project outside of the $\epsilon$ ball centered on the origin during learning as constrained by Eq.~\ref{eq:cones_closed_form} and its Euclidean version. For both we use $\epsilon=0.1$. A learning rate of 1e-4 is used for both Euclidean and hyperbolic cone models.

\paragraph{Results and discussion.}

Table~\ref{tab:results} shows the obtained results. For a fair comparison, we use models with the same number of dimensions. We focus on the low dimensional setting (5 and 10 dimensions) which is more informative. It can be seen that our hyperbolic cones are better than all the baselines in all settings, except in the $0\%$ setting for which order embeddings are better. However, once a small percentage of the transitive closure edges becomes available during training, we observe significant improvements of our method, sometimes by more than $8\%$ F1 score. Moreover, hyperbolic cones have the largest growth when transitive closure edges are added at train time. We further note that, while mathematically not justified\footnote{Indeed, mathematically, hyperbolic embeddings cannot be considered as Euclidean points.}, if embeddings of our proposed Euclidean cones model are initialized with the Poincar\'e embeddings instead of the Simple Euclidean ones, then they perform on par with the hyperbolic cones.

%
%
%
%

\section{Conclusion}
Learning meaningful graph embeddings is relevant for many important applications. Hyperbolic geometry has proven to be powerful for embedding hierarchical structures. We here take one step forward and propose a novel model based on geodesically convex entailment cones and show its theoretical and practical benefits. We empirically discover that strong embedding methods can vary a lot with the percentage of the taxonomy observable during training and demonstrate that our proposed method benefits the most from increasing size of the training data. As future work, it would be interesting to understand if the proposed entailment cones can be used to embed more complex data such as sentences or images.

Our code is publicly available\footnote{\url{https://github.com/dalab/hyperbolic_cones}.}.

\section*{Acknowledgements}
 We would like to thank Maximilian Nickel, Colin Evans, Chris Waterson, Marius Pasca, Xiang Li and Vered Shwartz for helpful discussions about related work and evaluation settings. 

This research is funded by the Swiss National Science Foundation (SNSF) under grant agreement number 167176. Gary B\'ecigneul is also funded by the Max Planck ETH Center for Learning Systems.

\bibliographystyle{icml2018}
\bibliography{th}

\newpage
\clearpage
%

\appendix

\section{Geodesics in the Hyperboloid Model} \label{sec:hyperboloid}

The hyperboloid model is $(\mathbb H^n, \inp{\cdot}{\cdot}_1)$, where $\mathbb H^n := \{x \in \mathbb \Re^{n,1}: \inp{x}{x}_1 = -1,\ x_{0} > 0\}$. The hyperboloid model can be viewed from the extrinsically as embedded in the pseudo-Riemannian manifold Minkowski space $(\Re^{n,1}, \inp{\cdot}{\cdot}_1)$ and inducing its metric. The Minkowski metric tensor $g^{\Re^{n,1}}$ of signature $(n,1)$ has the components
\begin{align*}
g^{\Re^{n,1}} = 
 \begin{bmatrix}
   -1 & 0 & \ldots & 0 \\
   0 & 1 & \ldots & 0 \\
   0 & 0 & \ldots & 0 \\
   0 & 0 & \ldots & 1 \\
   \end{bmatrix}
\end{align*}
The associated inner-product is $\inp{x}{y}_1 := -x_0y_0 + \sum_{i=1}^n x_iy_i$. Note that the hyperboloid model is a Riemannian manifold because the quadratic form associated with $g^{\mathbb H}$ is positive definite.

In the extrinsic view, the tangent space at $\HH^n$ can be described as $T_x\HH^n = \{ v \in \Re^{n,1} : \inp{v}{x}_1 = 0 \}$. See~\citet{robbin2011introduction,parkkonenhyperbolic}. \\

Geodesics of $\HH^n$ are given by the following theorem (Eq (6.4.10) in \citet{robbin2011introduction}):
\begin{theorem}
\label{thm:geodesics_hyperboloid}
Let $x \in \HH^n$ and $v \in T_x\HH^n$ such that $\inp{v}{v} = 1$. The unique unit-speed geodesic $\phi_{x,v}: [0,1] \rightarrow \HH^n$ with $\phi_{x,v}(0) = x$ and $\dot{\phi}_{x,v}(0) = v$ is
\begin{align}
\phi_{x,v}(t) = x \cosh(t) + v \sinh(t).
\end{align}
\end{theorem}

\section{Proof of Theorem~\ref{thm:geodesic_unit_speed}} \label{sec:appendix_thm1_proof}
\begin{proof}
From theorem \ref{thm:geodesics_hyperboloid}, appendix~\ref{sec:hyperboloid}, we know the expression of the unit-speed geodesics of the hyperboloid model $\HH^n$. We can use the Egregium theorem to project the geodesics of $\HH^n$ to the geodesics of $\D^n$. We can do that because we know an isometry $\psi : \D^n \rightarrow \HH^n$ between the two spaces:
\begin{align}
\psi(x) := (\lambda_x - 1, \lambda_x x), \quad \psi^{-1}(x_0, x') = \frac{x'}{1 + x_0}
\end{align}

Formally, let $x \in \D^n, v \in T_x\D^n$ with $g^\D(v,v) = 1$. Also, let $\gamma : [0,1] \rightarrow \D^n$ be the unique unit-speed geodesic in $\D^n$ with $\gamma(0) = x$ and $\dot{\gamma}(0) = v$. Then, by Egregium theorem, $\phi := \psi \circ \gamma$ is also a unit-speed geodesic in $\HH^n$. From theorem \ref{thm:geodesics_hyperboloid}, we have that $\phi(t) = x' \cosh(t) + v' \sinh(t)$, for some $x' \in \HH^n, v' \in T_{x'}\HH^n$. One derives their expression:
\begin{align}
& x' = \psi \circ \gamma (0) = (\lambda_x - 1, \lambda_x x) \\
& v' = \dot{\phi}(0) = \left. \frac{\partial \psi (y_0, y)}{\partial y}\right|_{\gamma(0)}  \dot{\gamma}(0) = 
\begin{bmatrix}
   \lambda_x^2 \inp{x}{v} \\
   \lambda_x^2 \inp{x}{v} x + \lambda_x v
   \end{bmatrix} \nonumber
\end{align}

Inverting once again, $\gamma(t) = \psi^{-1} \circ \phi (t)$, one gets the closed-form expression for $\gamma$ stated in the theorem. \end{proof}

One can sanity check that indeed the formula from theorem \ref{thm:geodesic_unit_speed} satisfies the conditions:
\begin{itemize}
\item $d_\D(\gamma(0), \gamma(t)) = t, \quad \forall t \in [0,1]$
\item $\gamma(0) = x$
\item $\dot{\gamma}(0) = v$
\item $\lim_{t \rightarrow \infty} \gamma(t) := \gamma(\infty) \in \partial \D^n$
\end{itemize}

\section{Proof of Corollary~\ref{cor:exp_map}}
\label{proof:exp_map}
\begin{proof}
Denote $u = \frac{1}{\sqrt{g_x^\D(v,v)}} v$. Using the notations from Thm.~\ref{thm:geodesic_unit_speed}, one has $\exp_x(v) = \gamma_{x, u}(\sqrt{g_x^\D(v,v)})$. Using Eq.~\ref{eq:metric_tensor} and~\ref{eq:geodesic_unit_speed}, one derives the result.
\end{proof}

\section{Proof of Corollary~\ref{cor:coplanar}}
\label{proof:coplanar}
\begin{proof}
For any geodesic $\gamma_{x,v}(t)$, consider the plane spanned by the vectors $x$ and $v$. Then, from  Thm.~\ref{thm:geodesic_unit_speed}, this plane contains all the points of $\gamma_{x,v}(t)$,  i.e.
\begin{align}
\{\gamma_{x,v}(t) : t \in \Re \} \subseteq \{ax + bv : a,b \in \Re \}
\end{align}
\end{proof}

\section{Proof of Lemma~\ref{lemma:acute_angle}} \label{proof:lemma_acute_angle}
\begin{proof}
Assume the contrary and let $x \in \D^n\setminus \{0\}$ s.t. $\psi(\|x\|) > \frac{\pi}{2}$. We will show that transitivity implies that
\begin{align}
\forall x' \in \partial \frak S^{\psi(x)}_x: \quad \psi(\|x'\|) \leq \frac{\pi}{2}
\label{eq:remaining_fact_lemma_acute}
\end{align}
If the above is true, by moving $x'$ on any arbitrary (continuous) curve on the cone border $\partial \frak S^{\psi(x)}_x$ that ends in $x$, one will get a contradiction due to the continuity of $\psi(\|\cdot\|)$.

We now prove the remaining fact, namely Eq.~\ref{eq:remaining_fact_lemma_acute}. Let any arbitrary $x' \in \partial \frak S^{\psi(x)}_x$. Also, let $y \in \partial \frak S^{\psi(x)}_x$ be any arbitrary point on the geodesic half-line connecting $x$ with $x'$ starting from $x'$ (i.e. excluding the segment from $x$ to $x'$). Moreover, let $z$ be any arbitrary point on the spoke through $x'$ radiating from $x'$, namely $z \in A_{x'}$ (notation from Eq.~\ref{eq:spoke}). Then, based on the properties of hyperbolic angles discussed before (based on Eq.~\ref{eq:angle-geodesics}), the angles $\angle yx'z$ and $\angle zx'x$ are well-defined. 

\begin{figure}[H]
  \center
    \includegraphics[width=0.3\textwidth]{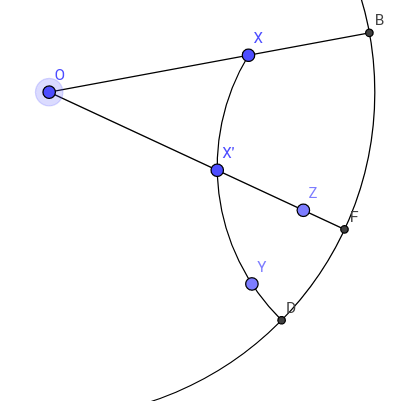}
\end{figure}

From Cor.~\ref{cor:coplanar} we know that the points $O, x, x', y, z$ are coplanar. We denote this plane by $\mathcal{P}$. Furthermore, the metric of the Poincar\'e ball is conformal with the Euclidean metric. Given these two facts, we derive that
\begin{align}
\angle yx'z + \angle zx'x = \angle(yx'x) = \pi
\end{align}
thus
\begin{align}
\min(\angle yx'z, \angle zx'x) \leq \frac{\pi}{2}
\end{align}

It only remains to prove that 
\begin{align}
\angle yx'z \geq \psi(x') \quad \& \quad \angle zx'x \geq \psi(x')
\label{eq:inequality_border}
\end{align}

Indeed, assume w.l.o.g. that $\angle yx'z < \psi(x')$. Since $\angle yx'z < \psi(x')$, there exists a point $t$ in the plane $\mathcal{P}$ such that 
\begin{align}
\angle Oxt < \angle Oxy \quad \& \quad \psi(x') \geq \angle tx'z > \angle yx'z
\end{align}
Then, clearly, $t \in \frak S^{\psi(x')}_{x'}$, and also $t \notin \frak S^{\psi(x)}_x$, which contradicts the transitivity property (Eq.~\ref{eq:cone_transitivity}).
\end{proof}

\section{Proof of Theorem~\ref{thm:sin_norm_product}} \label{proof:sin_norm_product}
\begin{proof}
We first need to prove the following fact:

\begin{lemma}
\label{lemma:helper}
Transitivity implies that for all $x\in\D^n\setminus \{0\}$, $\forall x' \in \partial \frak S^{\psi(x)}_x$:
\begin{align}
\sin(\psi(\|x'\|)) \sinh(\|x'\|_{\D}) \leq \sin(\psi(\|x\|)) \sinh(\|x\|_{\D}).
\end{align}
\end{lemma}
\begin{proof}
We will use the exact same figure and notations of points $y,z$ as in the proof of lemma~\ref{lemma:acute_angle}. In addition, we assume w.l.o.g that
\begin{align}
\angle yx'z \leq \frac{\pi}{2}
\label{eq:sin_norm_0}
\end{align}
Further, let $b \in \partial \D^n$ be the intersection of the spoke through $x$ with the border of $\D^n$. Following the same argument as in the proof of lemma~\ref{lemma:acute_angle}, one proves Eq.~\ref{eq:inequality_border} which gives:
\begin{align}
\angle yx'z \geq \psi(x')
\label{eq:sin_norm_1}
\end{align}
In addition, the angle at $x'$ between the geodesics $xy$ and $Oz$ can be written in two ways:
\begin{align}
\angle Ox'x = \angle yx'z
\label{eq:sin_norm_2}
\end{align}
Since $ x' \in \partial \frak S^{\psi(x)}_x$, one proves
\begin{align}
\angle Oxx' = \pi - \angle x'xb = \pi - \psi(x)
\label{eq:sin_norm_3}
\end{align}
We apply hyperbolic law of sines (Eq.~\ref{eq:sine_law}) in the hyperbolic triangle $Oxx'$:
\begin{align}
\frac{\sin(\angle Oxx')}{\sinh(d_\D(O,x'))} = \frac{\sin(\angle Ox'x)}{\sinh(d_\D(O,x))}
\label{eq:sin_norm_4}
\end{align}
Putting together Eqs.~\ref{eq:sin_norm_0},\ref{eq:sin_norm_1},\ref{eq:sin_norm_2},\ref{eq:sin_norm_3},\ref{eq:sin_norm_4}, and using the fact that $\sin(\cdot)$ is an increasing function on $[0, \frac{\pi}{2}]$, we derive the conclusion of this helper lemma.
\end{proof} 

\vspace{0.5cm}
We now return to the proof of our theorem. Consider any arbitrary $r, r' \in (0,1) \cap \text{Dom}(\psi)$ with $r < r'$. Then, we claim that is enough to prove that 
\begin{align}
\exists x \in \D^n, x' \in \partial \frak S^{\psi(x)}_x \quad s.t. \quad \|x\| = r , \|x'\| = r'
\label{remaining_fact_lemma3}
\end{align} 
Indeed, if the above is true, then one can use the fact~\ref{eq:poincare_norm}, i.e.
\begin{align}
\sinh(\|x\|_{\D}) = \sinh\left(\ln\left(\frac{1 + r}{1-r} \right)\right) = \frac{2r}{1-r^2}
\end{align}
and apply lemma~\ref{lemma:helper} to derive
\begin{align}
h(r') \leq h(r)
\end{align}
which is enough for proving the non-increasing property of function $h$. \\

We are only left to prove the fact~\ref{remaining_fact_lemma3}. Let any arbitrary $x \in \D^n$ s.t. $\|x\| = r $. Also, consider any arbitrary geodesic $\gamma_{x,v} : \R_+ \rightarrow \partial \frak S^{\psi(x)}_x$ that takes values on the cone border, i.e. $\angle(v,x) = \psi(x)$. We know that 
\begin{align}
\|\gamma_{x,v}(0)\| = \|x\| = r
\end{align}
and that this geodesic "ends" on the ball's border $\partial \D^n$, i.e.
\begin{align}
\|\lim_{t \rightarrow \infty}\gamma_{x,v}(t)\| = 1
\end{align}
Thus, because the  function $\|\gamma_{x,v}( \cdot )\|$ is continuous, we obtain that for any $r' \in (r,1)$ there exists an $t' \in \R_+$ s.t. $\|\gamma_{x,v}( t' )\| = r'$. By setting $x' := \gamma_{x,v}( t' ) \in \partial \frak S^{\psi(x)}_x$ we obtain the desired result.
\end{proof}

\section{Proof of Theorem~\ref{thm:equiv_def}}
\label{proof:equiv_def}
\begin{proof}
For any $y \in\frak S^{\psi(x)}_x$, the axial symmetry property implies that $\pi - \angle Oxy \leq \psi(x)$. Applying the hyperbolic cosine law in the triangle $Oxy$ and writing the above angle inequality in terms of the cosines of the two angles, one gets
\begin{align}
\cos \angle Oxy = \frac{- \cosh(\|y\|_\D) + \cosh(\|x\|_\D) \cosh(d_\D(x,y))}{\sinh(\|x\|_\D) \sinh(d_\D(x,y))} 
\end{align}
Eq.~\ref{eq:angle_oxy} is then derived from the above by an algebraic reformulation.
\end{proof}

\end{document}